\documentclass{tlp}

\title{A decidable subclass of finitary programs}
\author[S. Baselice, P.A. Bonatti]
         {Sabrina Baselice, Piero A. Bonatti\\
         Universit\`a di Napoli ``Federico II'', Italy\\
         }

  \usepackage{amsmath,amssymb}
  \usepackage{times}
  \usepackage{latexsym}
  \usepackage{graphicx}
  \usepackage{xspace}
  \usepackage{url}
  \usepackage{algorithm}
  \usepackage[noend]{algorithmic}

\def\proof{\relax{\noindent\it Proof.}}

   \newcommand{\nocc}{\mathsf{NOcc}}
   \newcommand{\smlr}{\preccurlyeq}
   \newcommand{\anlr}{\precsim}

  \newcommand{\hide}[1]{}
  \newcommand{\naf}{\ensuremath{\mathop{\tt{not}}}\xspace}

  \newcommand{\ground}{\mathsf{Ground}}

  \newcommand{\CHOOSE}{\STATE CHOOSE }

  \newcommand{\R}{\mathsf{R}}
  \newcommand{\U}{\mathsf{U}}
  
  \newcommand{\sups}{\mathsf{S}}
  \newcommand{\supp}{\mathsf{supp}}
  \newcommand{\DG}{\mathsf{DG}}
  \newcommand{\SCC}{\mathsf{SCC}}
  
  \newcommand{\ssup}{\mathsf{Ssup}}

  \newcommand{\mgu}{\mathsf{mgu}}

  \newcommand{\pred}{\mathsf{pred}}
  \newcommand{\head}{\mathsf{head}}
  \newcommand{\perm}{\mathsf{perm}}
  \newcommand{\body}{\mathsf{body}}
  \newcommand{\Def}{\mathsf{Def}}
  \newcommand{\Call}{\mathsf{Called}}
  \newcommand{\dep}{\rhd}
  \newcommand{\ind}{\|}
  
  \newcommand{\FP}{\ensuremath{\mathsf{FP}}\xspace}

  {
      \renewcommand{\theenumi}{\alph{enumi}}
      
      \begin{enumerate}
  }%
  {
      \end{enumerate}
  }

\newenvironment{enumroman}%
{
    \renewcommand{\theenumi}{\roman{enumi}}
    
    \begin{enumerate}
}%
{
    \end{enumerate}
}

\newcommand{\height}{\ensuremath{\mathsf{height}}\xspace}

\newtheorem{theorem}{Theorem}[section]
\newtheorem{lemma}[theorem]{Lemma}
\newtheorem{corollary}[theorem]{Corollary}
\newtheorem{proposition}[theorem]{Proposition}
\newtheorem{definition}[theorem]{Definition}
\newtheorem{example}[theorem]{Example}
\newtheorem{remark}[theorem]{Remark}

\begin{document}
\maketitle

\begin{abstract}
Answer set programming -- the most popular problem solving paradigm based on logic programs -- has been recently extended to support uninterpreted function symbols \cite{DBLP:conf/lpnmr/Syrjanen01,DBLP:journals/ai/Bonatti04,DBLP:conf/lpar/SimkusE07,DBLP:conf/lpnmr/GebserST07,DBLP:journals/tplp/BaseliceBC09,DBLP:conf/iclp/CalimeriCIL08}. All of these approaches have some limitation. 
In this paper we propose a class of programs called FP2 that enjoys a different trade-off between expressiveness and complexity. FP2 is inspired by the extension of finitary normal programs with local variables introduced in \cite[Sec.~5]{DBLP:journals/ai/Bonatti04}. FP2 programs enjoy the following unique combination of properties: (i) the ability of expressing predicates with infinite extensions; (ii) full support for predicates with arbitrary arity; (iii) decidability of FP2 membership checking; (iv) decidability of skeptical and credulous stable model reasoning for call-safe queries. Odd cycles are supported by composing FP2 programs with argument restricted programs.
\end{abstract}
\begin{keywords}
Answer set programming with function symbols, Infinite stable models, Norms.
\end{keywords}

	\section{Introduction}

Answer set programming has become the most popular problem solving paradigm based on logic programs. It is founded on the stable model semantics \cite{gelfond91classical} and supported by well-engineered implementations such as SMODELS \cite{smodels} and DLV \cite{DLV}, just to name a few.  Recent developments of the paradigm and its implementations include support for uninterpreted function symbols, pioneered by the work on finitary programs \cite{DBLP:journals/ai/Bonatti04,DBLP:journals/ai/Bonatti08}. These works gave rise to further developments, including argument restricted programs \cite{DBLP:conf/iclp/LierlerL09} and FDNC programs \cite{DBLP:conf/lpar/SimkusE07}, that address three limitations of finitary programs: a restriction on the number of odd-cycles in the dependency graph; the undecidability of the class of finitary programs; the dependency of reasoning on the set of odd-cycles, for which there is currently no general algorithm \cite{DBLP:journals/ai/Bonatti08}. The drawback of these approaches, in turn, is that either they cannot express predicates with infinite extensions such as the standard list and tree manipulation predicates \cite{DBLP:conf/iclp/CalimeriCIL08}, or they have to restrict predicate arity and rule structure in such a way that - roughly speaking - only models shaped like labelled trees can be characterized \cite{DBLP:conf/lpar/SimkusE07}. 

In this paper we propose a class of programs called FP2 that enjoys a different trade-off between expressiveness and complexity. FP2 is inspired by $U$-bounded programs (the extension of finitary normal programs with bounded local variables introduced in \cite[Sec.~5]{DBLP:journals/ai/Bonatti04}). FP2 programs retain the ability of expressing predicates with infinite extensions, and fully support predicates with arbitrary arity; moreover, deciding whether a program belongs to FP2 is decidable, as well as skeptical and credulous stable model reasoning, provided that the query is \emph{call-safe}. Odd cycles are supported by composing FP2 programs with argument restricted programs.

The paper is organized as follows.
After some preliminaries on logic programming, in Sec.~\ref{sec:norms} we introduce \emph{term comparison} relations based on a measure of term size called norm, and show how to compute those relations. In Sec.~\ref{sec:rec-analysis} we apply the term comparison relations to define \emph{recursion patterns}, that is, distinguished sets of arguments whose size almost never increases during recursion; we prove that if a recursion pattern exists then acyclic recursion depth is bounded, and there can be no odd-cycles. Then, in Sec.~\ref{FP2} we define FP2 and show that a form of SLD resolution with loop checking called \emph{acyclic derivations} suffice to compute a representative set of supports for each subgoal relevant to a given query. The output of this phase is a finite ground program that can be fed to an ASP solver to answer credulous and skeptical queries, in the same spirit as finitary programs. To re-introduce odd-cycles (and hence the ability to express constraints) in FP2, we show in Sec.~\ref{composition} how to compose FP2 programs with argument restricted programs. Two sections, on related work and a final discussion, conclude the paper.
Many proofs are omitted due to space limitations.

	\section{Preliminaries and notation}
		\label{sec:prelim}

We assume the reader to be familiar with classical logic 
programming \cite{DBLP:books/sp/Lloyd84}.
\emph{(Normal) logic programs} are finite sets of rules
    $ A \leftarrow L_1, ..., L_n \ 
  (n\geq 0) ,$ where $A$ is a logical
  atom and each $L_i$ ($i=1, ..., n$) is a \emph{literal}, that is,
  either a logical atom $B$ or a negated atom $\naf B$.
If $R$ is a rule with the above structure, then let $head(R)=A$ and $body(R)=\{L_1, ...,$ $L_n\}$.  
Moreover, let $body^+(R)$ (respectively $body^-(R)$) be the set of all atoms $B$
s.t. $B$ (respectively $\naf B$) belongs to $body(R)$.
%
For all predicate symbols $p$, a $p$-atom $A$ is an atom whose predicate, denoted by $\pred(A)$,
is $p$. Similarly a $p$-literal $L$ is a literal whose predicate, denoted by $\pred(L)$,
is $p$. 
The ground instantiation of a program $P$ is denoted by $\ground(P)$.
%
A Herbrand model $M$ of $P$ is a \emph{stable model} of $P$ iff $M$ is the least Herbrand model of $P^M$ and $P^M$ is the
\emph{Gelfond-Lifschitz transformation} of $P$ \cite{gelfond91classical}, obtained from
$\ground(P)$ by 
\hide{
\begin{enumroman}
\item removing all rules $R$ such that $body^-(R) \cap M \neq
  \emptyset$, and
\item  removing all negative literals from the body
  of the remaining rules.
\end{enumroman}
}
(i) removing all rules $R$ such that $body^-(R) \cap M \neq
  \emptyset$, and (ii) removing all negative literals from the body
  of the remaining rules.
%
A \emph{skeptical} consequence of a program $P$ is a formula satisfied
by all the stable models of $P$.  A \emph{credulous} consequence of
$P$ is a formula satisfied by at least one stable model of $P$.

The \emph{atom dependency graph of a program $P$} is a labelled directed
graph, denoted by $\DG_a(P)$, whose vertices are the ground atoms of
$P$'s language. Moreover,
\begin{enumroman}
\item there exists an edge labelled `+' (called positive edge) from
  $A$ to $B$ iff for some rule $R\in \ground(P)$, $A=\head(R)$ and
  $B\in body(R)$;
\item there exists an edge labelled `-' (called negative edge) from
  $A$ to $B$ iff for some rule $R\in \ground(P)$, $A=\head(R)$ and
  $\naf B\in body(R)$.
\end{enumroman}
An atom $A$ \emph{depends} on $B$ 
if there is a directed path from $A$ to $B$ in  $\DG_a(P)$.
\hide{
An atom $A$ \emph{depends positively} (respectively \emph{negatively}) on $B$ 
if there is a directed path from $A$ to $B$ in the atom dependency graph with an 
even (respectively odd) number of negative edges. Moreover, each atom depends 
positively on itself. $A$ \emph{depends} on $B$ if $A$ depends
positively or negatively on $B$.
}
Similarly, the \emph{predicate dependency graph of a program $P$} is a labelled directed
graph, denoted by $\DG_p(P)$, whose vertices are the predicate symbols of
$P$'s language. 
Edges are defined by analogy with the atom dependency graph.
\hide{ 

Moreover,
\begin{enumroman}
\item there exists an edge labelled `+' (called positive edge) from
  $p$ to $q$ iff for some rule $R\in P$, $\head(R)$ is a p-atom and
  $body(R)$ contains a positive q-literal;
\item there exists an edge labelled `-' (called negative edge) from
  $p$ to $q$ iff for some rule $R\in P$, $\head(R)$ is a p-atom and
  $body(R)$ contains a negative q-literal.
\end{enumroman}

} 
%
%
%
An \emph{odd-cycle} is a cycle in an atom (resp. predicate) dependency graph with an odd number 
  of negative edges. A ground atom (resp. a predicate symbol) is \emph{odd-cyclic} if it occurs 
  in an odd-cycle.
%
%
Given a graph $G$, we denote by $\SCC(G)$ the set of all strongly connected components in $G$. 
We say that a rule \emph{$R$ is in a component $C$} of a predicate dependency graph if 
$\pred(\head(R))$ is a vertex in $C$.


	\section{Norms and term comparisons}
	\label{sec:norms}
	
Norms have been introduced for the static termination analysis of logic programs, see for example \cite{DBLP:journals/tcs/BossiCF94,GCGL02}. Termination proofs require certain predicate arguments to decrease strictly during recursion; we admit cyclic programs, instead, and consider non-strict orderings.
For all sets of (possibly nonground) terms $t$, let $|t|$ (the
\emph{norm} of $t$) be the number of variables and function symbols
occurring in $t$ (constants are regarded as $0$-ary functions).  Norms
are extended to term sequences $\vec t = t_1,\ldots,t_n$ in the
natural way: By $|t_1,\ldots,t_n|$ we denote $|t_1|+ \cdots +|t_n|$.
For all vectors of terms $\vec t$ and $\vec u$, define $\vec t \smlr
\vec u$ (resp.\ $\vec t \prec \vec u$) iff for all grounding
substitutions $\sigma$, $|\vec t\sigma| \leq |\vec u\sigma|$ (resp.\
$|\vec t\sigma| < |\vec u\sigma|$).
Moreover, we write $\vec t \anlr \vec u$ iff $\vec t$ is \emph{almost
never larger than $\vec u$}, that is, there exist only finitely many
(possibly no) grounding substitutions $\sigma$ such that $|\vec
t\sigma| > |\vec u\sigma|$.
Note that $\vec t \prec \vec u \Rightarrow \vec t \smlr \vec u$ and
$\vec t \smlr \vec u \Rightarrow \vec t \anlr \vec u$.  
Note also that the norm over term sequences and the three comparison
relations are insensitive to permutations.  More precisely, for all
permutations $\vec t_1$ of $\vec t$, we have $|\vec t_1| = |\vec t|$;
therefore, if $\lessdot$ is any of the relations $\smlr$, $\prec$, and
$\anlr$, then for all $\vec u$, $\vec t_1 \lessdot \vec u
\Leftrightarrow \vec t \lessdot \vec u$ and $\vec u \lessdot \vec t_1
\Leftrightarrow \vec u \lessdot \vec t$.
All of these relations can be computed via simple variable occurrence counting.

\begin{theorem}    \label{thm:computing-rels}
Let $\nocc(s,\vec t)$ denote the number of occurrences of symbol $s$
in $\vec t$.
For all (possibly nonground) term vectors $\vec t$ and $\vec u$,
\begin{enumerate}
\item $\vec t \prec \vec u$ iff $|\vec t\,| < |\vec u|$ and for all
variables $x$, $\nocc(x,\vec t) \leq \nocc(x, \vec u)$;

\item $\vec t \smlr \vec u$ iff $|\vec t\,|\leq |\vec u|$ and for all
variables $x$, $\nocc(x,\vec t) \leq \nocc(x, \vec u)$;

\item $\vec t \anlr \vec u$ iff either $\vec t \smlr \vec u$ or for all
  variables $x$, $\nocc(x,\vec t) < \nocc(x, \vec u)$
\footnote{
  We cannot relax this condition.
  Indeed, let $\vec t = [X,Y,f(a)]$ and $\vec u = [X,Y,Y]$. 
  It holds that $\nocc(X,\vec t) \leq \nocc(X, \vec u)$
  and $\nocc(Y,\vec t) < \nocc(Y, \vec u)$. 
  However, if we set $Y=a$, for infinitely many substitutions $\sigma$ for $X$
  we have that $|[X,a,f(a)]\sigma| > |[X,a,a]\sigma|$.
}.
\end{enumerate}
\end{theorem}

\begin{proof}
It is easier to prove the contrapositive (which is equivalent).

1) First suppose that $|\vec t\,| \geq |\vec u|$. Then for all $\sigma$ mapping all variables onto constants, we have $|\vec t\,|=|\vec t\sigma\,|$ and $|\vec u|=|\vec u\sigma|$. This implies $|\vec t\sigma\,| \geq |\vec u\sigma|$, and hence $\vec t \prec \vec u$ does not hold. Second, if for some variable $x$, $\nocc(x,\vec t) > \nocc(x, \vec u)$, then there exists a $\sigma$ mapping all variables but $x$ onto constants, and mapping $x$ on a term with size $> |u|$. It is not hard to see that  $|\vec t\sigma\,| > |\vec u\sigma|$, and hence $\vec t \prec \vec u$ does not hold. This concludes the proof of point 1.
The proof of points 2 and 3 is based on analogous arguments.
\end{proof}

\hide{ 

A corollary of the above theorem shows that the relations $\smlr$,
$\prec$, and $\anlr$ can be used for groundness analysis.

\begin{corollary}               \label{cor:groundness}
Suppose that $\vec t \smlr \vec u$ or $\vec t \prec \vec u$ or $\vec t
\anlr \vec u$.  Then for all substitutions $\sigma$, if $\vec u\sigma$
is ground, then so is $\vec t\sigma$.
\end{corollary}

} 

\begin{example}
  Clearly $X \prec f(X) \smlr g(X)$. Moreover, $f(X,g(a)) \anlr f(X,Y)$, because $|f(X,g(a))\sigma| < |f(X,Y)\sigma|$ holds whenever $|Y\sigma|>2$ (for a finite program $P$, the set of terms with norm 1 or 2 is finite).  Finally, $f(X)$ and $f(Y)$ are incomparable.
\end{example}

	\section{Restricting recursion and odd-cycles}
	\label{sec:rec-analysis}

In FP2 programs recursion and odd-cycles are restricted, by analogy with finitary programs. This is partly achieved by requiring that for some groups of predicate arguments, norms should not increase ``too much'' during recursion.  Such groups of arguments are formalized via a suitable notion of argument \emph{selection indexes}.

An \emph{$n2k$-selection index} is a set
of distinct integers $a=\{a_1,\ldots,a_k\}$ such that $1\leq a_1 < a_2 < ...< a_k \leq n$.
An \emph{$n$-selection index} is any $n2k$-selection index.
By $- a$ we denote the \emph{complement} of an $n$-selection index $a$, that is, 
the set of integers between $1$ and $n$ that do not occur in $a$.
A selection index can be \emph{applied} to an atom to extract the corresponding arguments: for all atoms $A = p(t_1,\ldots,t_n)$ and $n2k$-selection
indexes $a$, define $A[a] = t_{a_1}, \ldots, t_{a_k}$.
Similarly, for all literals $L = \naf p(t_1,\ldots,t_n)$ and $n2k$-selection
indexes $a$, define $L[a] = t_{a_1}, \ldots, t_{a_k}$.

In FP2 programs each predicate is associated by a selection index to a group of arguments whose size almost never increases during recursion. Formally, a \emph{selection index mapping} for a program $P$ is a function $\mu$ mapping each $n$-ary predicate symbol $p$ in $P$ 
on an n-selection index $\mu_p$.
With a slight abuse of notation, if $p$ is the predicate occurring in an atom $B$ then we abbreviate $B[\mu_p]$ with $B[\mu]$.  
Similarly, if $L$ is a p-literal then $L[\mu]$ abbreviates $L[\mu_p]$.

We are only left to formalize two requirements: (i) the selected arguments should almost never increase during top-down evaluations, and (ii) there should be no odd-cycles, that in this context might be a symptom of the presence of infinitely many odd-cycles, thereby violating one of the essential properties of finitary programs.
A preliminary notion is needed first: a selection index mapping $\pi$ is \emph{complete} for a $n$-ary predicate symbol $p$ if $\pi$ maps $p$ on an $n2n$-selection index $\pi_p$. 


\begin{definition}
\label{def:decrule}
\begin{itemize}
\item 
  A rule $R$ in a program $P$ is \emph{decreasing} w.r.t.\ a selection index mapping $\pi$ for $P$ 
  iff for all literals $L\in\body(R)$ such that $\pred(L)$ and $\pred(\head(R))$ occur in the same 
  strongly connected component of $\DG_p(P)$,
	 $ L[\pi]\prec \head(R)[\pi] \,. $

\label{def:ani-rule}
\item
  A rule $R$ in a program $P$ is \emph{almost never increasing} w.r.t. a selection index mapping $\pi$ 
  if, for all literals $L\in\body(R)$ s.t. $\pred(L)$ and $\pred(\head(R))$ occur in the same 
  strongly connected component of $\DG_p(P)$, the following conditions hold: 
  \begin{enumerate}
  \item $L[\pi]\anlr \head(R)[\pi]$, and
  \item $\pi$ is complete for $\pred(L)$ and $\pred(\head(R))$.
  \end{enumerate}
\end{itemize}
\end{definition}

\begin{definition}
\label{def:rec-pat}
  A \emph{recursion pattern} $\pi$ for a program $P$ is a selection index mapping for $P$ s.t. 
  for each strongly connected component $C\in\SCC(\DG_p(P))$ at least one of the
following conditions holds: 
  \begin{enumerate}
   \item \label{defRecPatCond1} all $R\in C$ are decreasing w.r.t. $\pi$;
   \item \label{defRecPatCond2} all $R\in C$ are almost never increasing w.r.t. $\pi$ and $C$ does not contain any odd-cycles.
  \end{enumerate}
\end{definition}

\hide{ 

\begin{definition}[$\FP1$]
\label{def:FP1}
  A normal logic program $P$ belongs to the class $\FP1$ iff:    
  \begin{enumerate}
  \item \label{defFP1cond1} $P$ has a recursion pattern, and
  \item \label{defFP1cond2} $P$ contains no local variables.
  \end{enumerate}
\end{definition}

} 

As we anticipated, the existence of recursion patterns implies bounds on recursion depth and odd-cycle freedom.

\begin{lemma}
\label{lemma:rec-pat-finite-path}
  If a program $P$ has a recursion pattern then all paths in $\DG_a(P)$ contain finitely many different atoms.%
\footnote{It is not hard to see that if $\pi$ were not required to be complete for almost never increasing rules (cf.\ point 2 in Def.~\ref{def:decrule}) then this lemma would not be valid.}
\end{lemma}


\hide{ 

\begin{theorem}
\label{th:FP1inFR}
  All programs in $\FP1$ are finitely recursive.
\end{theorem}
\begin{proof}
  Similar to the proof of Theorem~\ref{th:FP0inFR} 
  (use Lemma~\ref{lemma:rec-pat-finite-path} instead of Lemma~\ref{lemma:strict-rec-pat-finite-path}
  and Definition~\ref{def:FP1} instead of Definition~\ref{def:FP0}).   
\end{proof}

} 

\begin{lemma}
\label{lemma:rec-pat-no-oddcycle}
  If a normal program $P$ has a recursion pattern then $\DG_a(P)$ is
  odd-cycle-free.
\end{lemma}

\hide{ 

\begin{proof}
  Suppose $P$ has a recursion pattern $\pi$.
  Let $s=A_1, \ldots, A_n, A_1$ be an odd-cyclic path in $\DG_a(P)$. As argued in the proof
  of Lemma~\ref{lemma:strict-rec-pat-finite-path}, all atoms in $s$ belong to a same strongly connected component $C$ in $\DG_p(P)$.

  If the condition~\ref{defRecPatCond1} of Definition~\ref{def:rec-pat} holds for $C$, 
  by Lemma~\ref{lemma:strict-rec-pat-acyclic}, $s$ cannot be cyclic.

  Now, suppose that $C$ satisfies the condition~\ref{defRecPatCond2} of Definition~\ref{def:rec-pat}.
  By replacing in $s$ each atom $A_i$ with its predicate symbol, we obtain a path $s'$ in $\DG_p(P)$ 
  that preserves the labels of the edges in $s$. Then $s'$ is an odd-cycle in $C$, and this contradicts
  the condition~\ref{defRecPatCond2} of Definition~\ref{def:rec-pat}.    
\end{proof}

} 


\begin{example}
\label{ex:append}
  Consider the classical program for appending lists:
  \[
  \begin{array}{ll}
  append([\,],L,L). & \hspace{0.5in}
  append([X|X_s],L,[X|Y_s]) \leftarrow append(X_s,L,Y_s).\\
  \end{array}
\]
Let $\mu$ be a selection index mapping for this program. If  $\mu_\mathit{append}=\{1\}$ then $\mu$ is a recursion pattern; indeed, the first rule is vacuously decreasing because it has an empty body, and the second rule is decreasing because the selected argument is decreasing: $X_s \prec [X|X_s]$.  Similarly, if  $\mu_\mathit{append}=\{3\}$ then $\mu$ is a recursion pattern. On the contrary $\mu_\mathit{append}=\{2\}$ does not yield a recursion pattern; we have $L \anlr L$, but $\mu_\mathit{append}$ is not complete. Finally, $\mu_\mathit{append}=\{1,2,3\}$ yields a recursion pattern (both rules are decreasing w.r.t.\ $\mu$).
\end{example}

\section{Acyclic derivations, supports, and stable models}
\label{Acyclic-derivations}

By analogy with the theory of finitary programs, a query $G$ over an FP2 program $P$ is answered by computing in a top-down fashion a representative, partially evaluated fragment of $\ground(P)$ that suffices to answer $G$. FP2 programs will be defined so that such top-down computations are finite and finitely many, therefore a complete enumeration thereof is possible.
Note that Lemma~\ref{lemma:rec-pat-finite-path} is not enough for this purpose for two reasons.
First, a loop-checking mechanism should be set up to avoid infinite cyclic derivations.
Second, there could still be infinitely many bounded, acyclic derivations (a situation that commonly arises in the presence of local variables, that make $\DG_p(P)$ infinitely branching). We shall constrain queries and rule bodies to be \emph{call safe} (see below) so that the selected, almost never increasing arguments of each predicate are bound whenever the predicate is called; we shall prove that, as a consequence, every subgoal yields finitely many answer substitutions that are all grounding.
In this section, we set up the technical machinery for the acyclic top-down computations which is based on annotating each goal with the history of previously resolved atoms in order to check for loops.

\subsection{Annotated and acyclic derivations}
\label{Annotated-derivations}

An \emph{annotated literal} (\emph{a-literal} for short) is a pair $L\alpha$ where $L$ is a literal and $\alpha$ is an \emph{annotation}, that is, a sequence of atoms. The empty annotation will be denoted with $\varepsilon$.
$L\alpha$ is \emph{positive} (resp.\ \emph{negative}) if $L$ is positive (resp.\ negative).
An \emph{annotated goal} (\emph{a-goal} for short) is a finite sequence $G=L_1\alpha_1,\ldots, L_n\alpha_n$ of annotated literals.
An annotated goal is \emph{cyclic} if some positive $L_i$ occurs in $\alpha_i$\,, \emph{acyclic} otherwise.

Given an a-goal $G=L_1\alpha_1,\ldots, L_n\alpha_n$, a positive $L_i\alpha_i$ in $G$ $(1\leq i\leq n)$, and a rule $R=A \leftarrow L'_1,\ldots,L'_m$ such that $L_i$ and $A$ are unifiable and $\mgu(L_i,A)=\theta$, the goal 
\[
   \big( L_1\alpha_1,\ldots, L_{i-1}\alpha_{i-1},L'_1\alpha',\ldots, L'_m\alpha',L_{i+1}\alpha_{i+1},\ldots, L_n\alpha_n \big) \theta 
\]
where $\alpha'=L_i\cdot\alpha_i$ is called the \emph{annotated resolvent} of $G$, $L_i$, and $R$ with mgu $\theta$.  The atom $L_i$ is called \emph{selected atom}, and in this paper it will always be the \emph{leftmost positive literal} of $G$. Accordingly, the selected literal will frequently be omitted.

An \emph{annotated derivation} (\emph{a-derivation} for short) of $G_0$ from a program $P$ with rules $R_1,\ldots,R_i,\ldots$ and mgu's $\theta_1,\ldots,\theta_i,\ldots$ is a (possibly infinite) sequence of a-goals\break 
$G_0,\ldots,G_i,\ldots$ such that each $G_j$ in the sequence with $j>0$ is the annotated resolvent of $G_{j-1}$ and $R_j$ with mgu $\theta_j$\,, for some standardized apart variant $R_j$ of a rule in $P$.
An a-derivation is \emph{acyclic} if all of its a-goals are, possibly with the exception of the last goal if the derivation is finite. Intuitively, an acyclic derivation \emph{fails} as soon as a cycle is detected.

\hide{ 
The \emph{length} of a finite a-derivation $G_0,\ldots, G_n$ is $n$.  The length of an infinite a-derivation is the ordinal $\omega$.

A  \emph{failed} a-derivation is either an infinite derivation or a finite failed derivation; a finite a-derivation is \emph{failed} if the leftmost positive literal of its last a-goal unifies with no rule in $P$.
} 
An a-derivation is \emph{successful} if it is finite and its last element contains no positive a-literals.
If $G_0,\ldots, G_n$ is a successful a-derivation with mgu's
$\theta_1,\ldots, \theta_n$, then we call the composition
$\theta^g=\theta_1\circ\cdots\circ \theta_n$ a \emph{global answer} to
$G_0$, and the restriction of $\theta^g$ to the variables of $G_0$ an
\emph{answer substitution} to $G_0$.


\begin{example}
  Consider the program $P$ consisting of the rules $p(X) \leftarrow q(X)$, $q(X) \leftarrow p(X)$, and $p(a)$. The goal $p(a)$ has both a successful acyclic a-derivation $p(a)\varepsilon,\Box$ (where $\Box$ denotes the empty goal) where $p(a)$ is resolved with the third rule, and a failed acyclic derivation using the first two rules: $p(a)\varepsilon,\, q(a)p(a),\, \underline{p(a)}(q(a)\cdot\underline{p(a)})$. The underlined literals show that the last goal is cyclic.
\end{example}

The main results of the paper will need the following technical definitions and lemmata.

Let $G_0$ be an a-goal with at least $k$ positive a-literals. 
An \emph{embedded a-derivation of degree $k$} for $G_0$ is an a-derivation $\Delta=G_0,G_1,\ldots$ such that for some suffix $G''$ of $G_0$:
\begin{itemize}
\item for all $i=0,1,\ldots$, it holds $G_i=G'_iG''$, for some $G'_i$;
\item the number of positive a-literals in $G'_0$ is $k$;
\item if $\Delta$ is finite and $G_n$ is its last goal, either $G'_n$ has no positive a-literals or $G'_n$ is failed; 
in the former case, the embedded a-derivation is \emph{successful}, otherwise it is \emph{failed}.
\end{itemize}

\hide{ 

\begin{lemma}[Lifting]
  \label{lifting-lemma}
  For all a-derivations $G_0,\ldots, G_i,\ldots$ of $G_0$ from $\Ground(P)$ with rules $R_1,\ldots, R_i,\ldots$ 
  and mgu's $\theta_1,\ldots,\theta_i,\ldots$ there exists an a-derivation $G'_0,\ldots, G'_i,\ldots$ 
  of $G_0$ from $P$ with rules $R'_1,\ldots, R'_i,\ldots$ and mgu's $\theta'_1,\ldots,\theta'_i,\ldots$ 
  such that the two derivations have the same length $\ell$, and  for all integers $0< j\leq \ell$, $G_j$ 
  and $R_j$ are ground instances of $G'_j$ and $R'_j$, respectively, and $\theta'_j$ is more general than $\theta_j$.
\end{lemma}
\begin{proof}
  Similar to the proof of the lifting lemma for positive logic programs \cite{Ap88}. 
\end{proof}

} 

\noindent
Intuitively, an embedded derivation of degree $k$, if successful, resolves the first $k$ positive literals of the initial goal.
In the following sections, we will sometimes split derivations into multiple embedded derivations to apply the induction hypotheses.  The following two lemmata help.

\begin{lemma}[Decomposition 1]
  \label{decomposition-1}
  $\Delta$ is an embedded a-derivation of $G_0$ from $P$ with degree $1$ iff either 
  (i) $G_0$ is failed and $\Delta=G_0$, 
  or (ii) $G_0$ has an annotated   resolvent $G_1$ with rule $R$ and mgu $\theta$, 
  and $\Delta = G_0\cdot\Delta'$, where $\Delta'$ is an embedded a-derivation of $G_1$ with degree $k$
  and $k$ is the number of positive literals in the body of $R\theta$.
\end{lemma}

\hide{ 

\begin{proof}
  \begin{description}
  \item [Length of $\Delta$ is 0:] Let $\Delta=G_0$.  
  Since $\Delta$ is an embedded a-derivation, the a-goal $G_0$ contains 1 positive a-literal $A\alpha$ 
  iff the atom $A$ unifies with no rule in $P$ and then iff $G_0$ is failed ($A\alpha$ is the leftmost positive a-literal of $G_0$).
  \item [Length of $\Delta$ is greater than $0$:] Let $\Delta=G_0, G_1,G_2,...$. 
  For all $i=0,1,...$, it holds $G_i=G'_iG''$, for some $G'_i$ and $G''$. 
  Since the length of $\Delta$ is greater than $0$, the goal $G_0$ contains at least 1 positive a-literal.
  Let $G_0'$ contain 1 positive a-literal.  
  The a-goal $G_1$ is an annotated resolvent of $G_0$ with rule $R$ and mgu $\theta$,
  iff $G'_1$ contains $k$ positive a-literals, where $k$ is the number of positive literals in the body of $R\theta$, 
  and then iff $\Delta=G_0\Delta'$ where $\Delta'=G_1,G_2,...$ is an embedded a-derivation of $G_1$ with degree $k$.
  \end{description}
\end{proof}

} 

Given two a-derivations $\Delta=G_0,\ldots, G_m$ and
$\Delta'=G'_0,\ldots, G'_n$ such that $G_m=G'_0$, the \emph{join} of
$\Delta$ and $\Delta'$ is $G_0,\ldots, G_m,G'_1,\ldots, G'_n$.

\begin{lemma}[Decomposition 2]
  \label{join-embedded}
  $\Delta$ is an embedded a-derivation of $G_0$ from $P$ with degree
  $k$ iff either (i) $\Delta$ is the join of an embedded a-derivation
  $\Delta'$ of $G_0$ from $P$ with degree $1$ and an embedded
  a-derivation of $G_n$ from $P$ with degree $k-1$, where $G_n$ is the
  last a-goal of $\Delta'$; or (ii) $\Delta$ is an infinite embedded
  a-derivation for $G_0$ of degree 1.
\end{lemma}

\hide{ 

\begin{proof}
  Since $\Delta$ is an embedded a-derivation, for all $i=0,1,...$, it holds $G_i=G'_iG''$ for some $G''$.

  Let $G_0=G^1_0G^{k-1}_0G''$ and $G^1_0G^{k-1}_0=G'_0$, where $G^1_0$ contains only 1 positive a-literal 
  and $G^{k-1}_0$ contains $k-1$ positive a-literals.
  Suppose that there is an a-goal $G_n=G'_nG''$ in $\Delta$ s.t. $G^{k-1}_0$ is a suffix of $G'_n$ and 
  all positive a-literals of $G'_n$ are in $G^{k-1}_0$.
  Then, $\Delta'=G_0,...,G_n$ is an embedded a-derivation for $G_0$ of degree 1 and $\Delta''=G_n,G_{n+1},...$
  is an embedded a-derivation for $G_n$ of degree $k-1$.

  If there is no such an a-goal $G_n$ then, for all $i=0,1,...$, the goal $G^{k-1}_0$ is a suffix of $G'_i$
  and $G'_i$ contains some positive a-literals. 
  Then, $\Delta$ is an infinite embedded a-derivation for $G_0$ of degree 1.
\end{proof}

} 

	\subsection{Finiteness and groundness properties of call-safe, acyclic a-derivations}
	\label{acyclic-reasoning}

The good finiteness and termination properties we need acyclic derivations to enjoy in order to prove the termination of our algorithms can be enforced by a ``call safeness'' property that ensures that the arguments selected by recursion pattern are always bound when a predicate is called.

\begin{definition}[Call-safeness]
\label{def:callSafeness}
  An a-goal  $L_1\alpha_1, ...,L_n\alpha_n$ is \emph{call-safe} w.r.t.\ a selection index mapping $\mu$ iff for all variables $X$ occurring in some $L_i[\mu]$ or in a negative literal $L_i$  ($1\leq i \leq n$), $X$ occurs also in a positive literal $L_j$, with $1\leq j<i$.
Similarly, a rule $R: A\leftarrow L_1, L_2, ..., L_n$ is \emph{call-safe} w.r.t.\ $\mu$ iff for each variable $X$ occurring in $R$,
  some of the following conditions hold:
  \begin{enumerate}
  \item \label{cond1DefCallPattern} $X$ occurs in $A[\mu]$;
  \item \label{cond2DefCallPattern} $X$ occurs in $\body(R)$; moreover, if $X$ occurs in $L_i[\mu]$ or in a negative literal $L_i$, for some $i=1,\ldots,n$, 
  then $X$ occurs also in a positive literal $L_j$, with $1\leq j<i$.
  \end{enumerate}
Finally, a program $P$ is \emph{call-safe} w.r.t.\ $\mu$ iff for all $R\in P$, $R$ is call-safe w.r.t.\ $\mu$.
\end{definition}

\begin{example}
  Consider the following program for reversing a list:
  \[
  \begin{array}{ll}
    \mathit{reverse}([\,],[\,]). & \hspace{0.5in}
    \mathit{reverse}([X| Y],Z) \leftarrow 
       \mathit{reverse}(Y,W), \mathit{append}(W,[X],Z). \\
  \end{array}
  \]
  The first rule is trivially call safe w.r.t.\ any selection index mapping.
  If $\mu_\mathit{reverse}=\{1\}$ and $\mu_\mathit{append}=\{1\}$ then
  the second rule is call-safe w.r.t.\ $\mu$.
  To see this, note that: (i) $X$ and $Y$ satisfy condition (\ref{cond1DefCallPattern}); (ii) $Z$ satisfies (\ref{cond2DefCallPattern}) because it occurs in the body but not in any selected argument nor in any negative literal; (iii) $W$, the selected argument of the second subgoal, satisfies condition (\ref{cond2DefCallPattern}) because it occurs also in the first subgoal.
\end{example}

If both $P$ and $G$ are call-safe, then call-safeness is preserved along all the steps of a derivation:

\begin{lemma}
  \label{call-safeness}
  Let $P$ be a normal logic program and $G$ an a-goal. If $P$ and $G$
  are call-safe w.r.t.\ a selection index mapping $\mu$, then all
  resolvents of $G$ and a rule $R\in P$ are call-safe w.r.t.\ $\mu$, too.
\end{lemma}

\begin{proof}
  Let $G'$ be an annotated resolvent of the first positive a-literal $L\alpha$ in $G$ and 
  a rule $R=A\leftarrow L_1,...,L_k$ in $P$ with a substitution $\theta=\mgu(L,A)$. 
  Since $G$ is call-safe, $L[\mu]$ is ground, and so must be $A\theta[\mu]$. It follows -- since $P$ is call-safe -- that the a-goal $L_1(L\cdot\alpha),...,L_k(L\cdot\alpha)$ must be call-safe w.r.t. $\mu$.
  Then $G'$ is call-safe w.r.t. $\mu$, too.
\end{proof}

\hide{ 

\begin{proof}
  Let $G'$ be an annotated resolvent of $G$, the first positive a-literal $L\alpha$ in $G$ and 
  a rule $R=A\leftarrow L_1,...,L_k$ in $P'\in\perm(P,\tau)$ with a substitution $\theta=\mgu(L,A)$. 
  Since $G$ is call-safe, $L[\tau]$ and $A\theta[\tau]$ are ground and, 
  since $P'\in\perm(P,\tau)$, the a-goal $L_1(L\cdot\alpha),...,L_k(L\cdot\alpha)$ is call-safe w.r.t. $\tau$.
  Then the resolvent $G'$ is again call-safe w.r.t. $\tau$.
\end{proof}

} 

\noindent
Furthermore, the binding propagation schema imposed by call-safeness ensures that global answers are grounding:

\begin{lemma}
  \label{ground-answers}
  Let $P$ be a normal logic program and $G_0$ an a-goal. Assume that
  $P$ and $G_0$ are call-safe w.r.t.\ a selection index mapping $\mu$.
  If $G_0$ has a successful a-derivation $\Delta=G_0,...,G_n$ from $P$ with
  global answer $\theta$, then for all $i=0,...,n$, the a-goal
  $G_i\theta$ is ground.
\end{lemma}


\begin{proof}
  By induction on the length of  $\Delta$.
  \begin{description}
  \item [Base case (the length of $\Delta$ is 0):] Let $\Delta=G_0$. 
  Since $G_0$ is call-safe, all its a-literals have to be negative and ground. Then, $G_0\theta$ is ground.    
  \item [Inductive step (the length of $\Delta$ is $n+1$):] Let $\Delta=G_0, G_1, ..., G_{n+1}$. 
  By Lemma~\ref{call-safeness}, $G_1$ is call-safe w.r.t. $\mu$. 
  Moreover, $\Delta'=G_1, ..., G_{n+1}$ is a successful a-derivation of length $n$ for $G_1$ 
  with global answer $\theta'$ more general than $\theta$. 
  By inductive hypohesis, for all $i=1,...,n+1$, the a-goal $G_i\theta'$ is ground. 
  Since $G_1$ is a resolvent of $G_0$ and both goals are call-safe, 
  all variables in $G_0\theta$ must be also in $G_1\theta'$. Consequently, $G_0\theta$ is ground.
  \end{description}
  \vspace*{-1.5em}
\end{proof}


The proof of the main theorem -- that we need to prove the termination of our reasoning algorithm -- will be based on inductions over the three indices defined below.

Let the height of a predicate $q$ be the cardinality of the set of predicates reachable from $q$ in $DG_p(P)$. The \emph{height} of an atom $A$ is the height of $\pred(A)$. Note that (i) $\height(A)\geq 1$; (ii) if $\pred(A)$ depends on $\pred(A')$ but not viceversa, then $\height(A) > \height(A')$; (iii) if $\pred(A)$ and $\pred(A')$ belong to the same strongly connected component of $DG_p(P)$, then $\height(A) = \height(A')$. By convention, the height of a negative literal is $0$.

Let $\pi$ be a recursion pattern for $P$, and $A$ be an atom such that $A[\pi]$ is ground. The \emph{call size} of $A$ is $|A[\pi]|$ (the norm of $A[\pi]$).

For the strongly connected components $C$ of $DG_p(P)$ in which the call size does not decrease during recursion, 
we adopt a ``loop saturation'' index. 
Let a \emph{$C$-atom} be an atom $A$ with $\pred(A)\in C$. 
Given a ground $C$-atom $A$, let $\max_{A}$ be the number of ground $C$-atoms $B$ such that $A$ depends on $B$. Such an integer $\max_{A}$ exists due to the following lemma (that shows why $\pi$ should be complete over such $C$): 

\begin{lemma}
\label{cycle-bound}
  Let $P$ be a program with a recursion pattern $\pi$ and $C$ be a strongly connected component of $DG_p(P)$
  s.t. $\pi$ is complete for predicates in $C$.
  Every ground $C$-atom $A$ depends on finitely many $C$-atoms in $P$.
\end{lemma}

\hide{ 

\begin{proof}
  By Lemmata~\ref{lemma:strict-rec-pat-finite-path} and~\ref{lemma:rec-pat-finite-path}, 
  all paths in $\DG_a(P)$ contain finitely many different atoms.
  Moreover, the $C$-atoms in the rules whose head predicate is in $C$ 
  do not contain local variables because $\pi$ is complete for their predicate symbols and then
  only finitely many different $C$-atoms are adjacent to a $C$-atom in $\DG_a(P)$. 
\end{proof}

} 

\noindent
Clearly $\max_A$ is an upper bound to the number of consecutive ground $C$-atoms occurring in an acyclic annotation. 
The \emph{loop saturation} index of an a-literal $A\alpha$ is $\max_{A}$ minus the length of the longest prefix of $\alpha$ consisting of $C$-atoms only. 
If the loop saturation index of $A\alpha$ is $0$, then every resolvent of $A\alpha$ that contains a $C$-atom is cyclic; if the loop saturation index is $\ell>0$, then all the $C$-atoms in the resolvents of $A\alpha$ have loop saturation index $\ell-1$.

\begin{theorem}[Strong finiteness]
  \label{embedded-termination}
  Let $P$ be a program with a recursion pattern $\pi$. Let $G_0$ be an
  a-goal with $k$ or more positive a-literals. Assume that $P$ and
  $G_0$ are call-safe w.r.t.\ $\pi$.  Then $G_0$ has finitely many
  acyclic embedded a-derivations of degree $k$ from $P$. Moreover,
  they are all finite.
\end{theorem}

\begin{proof}
  By induction on the maximum height of the literals in $G_0$. The base case is trivial.
\hide{ 

: if the maximum height is $0$, then all literals in $G_0$ are negative and the unique a-derivation of $G_0$ contains only $G_0$ itself.

} 
Now assume that the theorem holds for all heights $\leq n$; let $A_1\alpha_1,\ldots,A_k\alpha_k$ be the first $k$ positive a-literals of $G_0$, and assume that the maximum height of $A_1,\ldots,A_k$ is $n+1$.

We first prove the theorem for ``homogeneous'' cases where the atoms with maximum height belong to a same strongly connected component $C$ of $DG_p(P)$, that is, the members of $\{A_1,\ldots,A_k\}$ with height $n+1$ are all $C$-atoms.
This case is further divided in two subcases:
\begin{description}
\item[SC1] all rules $R\in C$ are decreasing w.r.t.\ the recursion pattern $\pi$;

\item[SC2] all rules $R\in C$ are almost never increasing w.r.t.\  $\pi$.
\end{description}

\emph{Proof of SC1}. By induction on the maximum call size of the  members of $\{A_1,\ldots,A_k\}$ with height $n+1$. 

\emph{Base case for SC1} (the maximum call size is $0$). By induction on $k$. 
If $k=1$, then for all resolvents $G_1$ of $G_0$ with rule $R$ and mgu $\theta$, consider the positive literals $A'_1\ldots A'_j$ in the body of $R\theta$. Since the call size of $A_1$ is $0$, $R$ is decreasing w.r.t.\ $\pi$, and the call size is non-negative, it follows that the height of $A'_1\ldots A'_j$ must be smaller than $n+1$. By the induction hypothesis relative to height, the embedded a-derivations for $G_1$ of degree $j$ are finite and finitely many. Then the same property holds for the embedded a-derivations for $G_0$ of degree $1$, by Lemma~\ref{decomposition-1}. This completes the proof for $k=1$.

Now assume $k>1$. By Lemma~\ref{join-embedded}, every embedded a-derivations for $G_0$ of degree $k$ is the join of two embedded a-derivations of degree $1$ and $k-1$, respectively.  Then the theorem easily follows from the induction hypothesis for degrees $k'<k$.

\emph{Induction step for SC1} (the maximum call size is $c>0$). The proof is similar to the proof of the base case.  The only difference is that the the positive literals $A'_1\ldots A'_j$ in the body of $R\theta$ may belong to $C$ and have degree $n+1$, however their maximum call size must be smaller than $c$ because $R$ is decreasing w.r.t.\ $\pi$.  Then it suffices to apply the induction hypothesis relative to the maximum call size instead of the one relative to height.
This completes the proof of SC1.

\emph{Proof of SC2}. %
Analogous to the proof of SC1.
In this case the induction is on the maximum loop saturation index $\ell$ of the first positive a-literals $A_1\alpha_1,\ldots,A_k\alpha_k$.
Call safeness, Lemma~\ref{call-safeness}, and the completeness of $\tau$ over $C$ ensure that all $C$-atoms occurring in the derivation are ground, so that the loop saturation index is well-defined. Details are omitted due to space limitations.
\hide{ 

\emph{Base case for SC2} (the maximum loop saturation index is $0$). By induction on $k$.
If $k=1$, then for all resolvents $G_1$ of $G_0$ with rule $R$ and mgu $\theta$, either $G_1$ is cyclic (and the acyclic a-derivation cannot be further extended) or all the positive literals $A'_1\ldots A'_j$ in the body of $R\theta$ have height smaller than $n+1$. By analogy with the base case for SC1, it follows from the induction hypothesis on height and Lemma~\ref{decomposition-1} that the theorem holds for $k=1$.  Similarly, the proof for $k>1$ follows from the induction hypothesis on $k$ and Lemma~\ref{join-embedded}.

\emph{Induction step for SC2} (the maximum loop saturation index is $\ell>0$). The proof is similar to the base case.  The only difference is that the rewriting of $A_1\alpha_1$ may produce literals $A'_i(A_1\cdot\alpha_1)$ with height $n+1$ but loop saturation $\ell-1$.  Then it suffices to replace the application of the induction hypothesis relative to height with an application of the induction hypothesis relative to $\ell$.

} 
This completes the proof of the homogeneous case.

Finally, we are left to prove the theorem for ``non homogeneous'' goals where the atoms with maximum height among $A_1,\ldots,A_k$ may belong to different strongly connected components.
The proof is by induction on $k$; the induction step relies on Lemma~\ref{join-embedded}.
\hide{ 

The proof is by induction on $k$.  If $k=1$, then we have an instance of the previous ``homogeneous'' case, and the theorem is proved. If $k>1$, then by Lemma~\ref{join-embedded}, every embedded a-derivations for $G_0$ of degree $k$ is the join of two embedded a-derivations of degree $1$ and $k-1$, respectively.  Then the theorem easily follows from the induction hypothesis for degrees $k'<k$.

} 
\end{proof}


\begin{corollary}
  \label{strong-termination}
  Let $P$ be a program with a recursion pattern $\pi$ and $G_0$ be an
  a-goal. Assume that $P$ and $G_0$ are call-safe w.r.t.\ $\pi$.  Then
  $G_0$ has finitely many acyclic a-derivations from $P$. Moreover,
  they are all finite.
\end{corollary}

\hide{ 

\begin{proof}
  Note that all successful a-derivations are successful embedded a-derivations of degree $k$, where $k$
  is the number of positive a-literals in $G_0$, and viceversa.
\end{proof}

} 

	\subsection{Acyclic supports and stable models}
	\label{A-derivations-vs-P-proofs}

It is well-known that the stable models of a program $P$ are completely characterized by the supports of $P$'s ground atoms.  In our setting, 
a \emph{support} for an a-goal $G$ with answer substitution $\theta$ is a set of negative literals $\{L_1,\ldots,L_n\}$ such that $G$ 
has a successful a-derivation from $\ground(P)$ with answer substitution $\theta$ and last goal $L_1\alpha_1,\ldots,L_n\alpha_n$. 
A support for an atom $A$ is a support for the a-goal $A\varepsilon$.  
The set of (negative) literals occurring in the last a-goal of a successful a-derivation $\Delta$ is called the \emph{support of $\Delta$}.
By \emph{acyclic support} we mean a support generated by an acyclic derivation.
The first result tells that by adopting acyclic a-derivations, only redundant supports can be lost:

\begin{theorem}[Completeness of acyclic derivations w.r.t.\ supports]
  \label{acyclic-enough}
  If $G_0$ has a successful a-derivation $\Delta$ from $P$ with global answer $\theta$, 
  then $G_0$ has a successful acyclic a-derivation $\Delta^a$ from $P$ with global answer $\theta^a$ 
  such that $\theta^a$ is more general than $\theta$ and the support of $\Delta^a$ is more general than a subset of the support of $\Delta$.
\end{theorem}


A-derivations and the related notion of support are in close correspondence with the P-proofs of \cite{DBLP:conf/iclp/MarekR08} and the corresponding supports.  By exploiting these relationships and the previous lemma, one can easily prove the following characterization of stable models in terms of the supports of acyclic a-derivations.


\begin{theorem}
\label{th:support-vs-models}
  Let $P$ be a normal program. A set $M$ of ground atoms is a stable
  model of $P$ iff $M$ is the set of all ground atoms that have a
  ground acyclic support in $\ground(P)$ satisfied by $M$.
\end{theorem}

\hide{ 

\begin{proof}
  It follows from Theorem~\ref{th:P-proof-vs-support} and Proposition~\ref{prop:marek-support}.
\end{proof}

} 

	\section{The class FP2}
	\label{FP2}

We are finally ready to introduce the class of FP2 programs.
If $\pi$ and $\tau$ are two selection index mappings, we say that  $\tau$ \emph{contains} $\pi$ (in symbols,  $\tau\supseteq\pi$) iff, for each predicate symbol $p$, it holds that $\tau_p\supseteq\pi_p$.

\begin{definition}[Call patterns]
\label{def:callPattern}
  A selection index mapping $\tau$ for a normal program $P$ is a
  \emph{call pattern} for $P$ iff (i) $\tau$ contains a recursion
  pattern of $P$, and (ii) for each rule $R\in P$ there exists a
  permutation $L_1, L_2, ..., L_n$ of $\body(R)$ such that \hide{the
  permuted rule} $\head(R)\leftarrow L_1, L_2, ..., L_n$ is call-safe
  w.r.t.\ $\tau$.
\end{definition}
\vspace*{-.7em}
\begin{definition}[$\FP2$]
\label{def:FP2}
  A normal logic program belongs to the class $\FP2$ iff it has a call pattern.
\end{definition}
\vspace*{-.7em}
\begin{example}
  The append program of Example~\ref{ex:append} is in FP2. It is easy to verify that if $\tau_\mathit{append}=\{3\}$ then $\tau$ is not only a recursion pattern (see Ex.~\ref{ex:append}), but also a call pattern.  On the contrary, the recursion pattern yielded by $\tau_\mathit{append}=\{1\}$ is not a call pattern because the variable $L$ in the first rule occurs neither in the selected argument (the first one) nor in the body. However this recursion pattern is contained in two call patterns, defined by $\tau_\mathit{append}=\{1,2\}$ and $\tau_\mathit{append}=\{1,3\}$.
\end{example}
For an example of a cyclic FP2 program with negation see the blocks world program in \cite[Fig.4]{DBLP:journals/ai/Bonatti04}. To make it an FP2 program, uniformly replace $T+1$ with $T$ in the second arguments of predicate $\mathit{ab}$. Then the (unique) selection index that is complete for all predicates is a call pattern for the program.

\hide{
\begin{lemma}
\label{lemma:decrule}
  If $P$ is a $\FP2$ program with a call pattern $\tau$ and $R\in P$ is a rule s.t.
  $\tau$ is not complete for $\pred(\head(R))$ then $R$ is decreasing w.r.t.
  a selection index mappng $\pi\subseteq\tau$.
\end{lemma}
\begin{proof}
  By definition of $\FP2$ programs, there is a selection index mapping $\pi\subseteq \tau$ 
  s.t.each rule in $P$ belonging to a strongly connected component $C$ in $\DG_p(P)$ 
  is either decreasing or almost never increasing w.r.t. $\pi$. By Definition~\ref{def:decrule},
  if no predicate symbol in $\body(R)$ belongs to $C$ then $R$ is trivially decreasing w.r.t.
  $\pi$, otherwise $R$ cannot be almost never increasing w.r.t. $\pi$ because $\pi$ is not complete 
  for $\pred(\head(R))$. Then, $R$ must be decreasing w.r.t. $\pi$.
\end{proof}
}

In general, it may be necessary to use different call patterns for different initial goals, in order to satisfy call safeness. In the above example a goal $\mathit{append}(t_1,t_2,t_3)$ is call safe w.r.t.\ some call pattern iff either $t_3$ is ground or at least two arguments are ground; different situations require different call patterns.

\hide{ 

In general, it may be necessary to use different call patterns for different initial goals. Rule bodies shall be permuted accordingly, in order to satisfy call safeness. This is formalized via the following definition:
Let $P$ be a normal logic program and let $\tau$ be a call pattern for $P$.
We denote by $\perm(P,\tau)$ the set of programs $P'$ such that $P'$ is call-safe w.r.t.\ $\tau$ and $P'$ is obtained from $P$ by permuting rule bodies.
Note that $P$ and all programs in $\perm(P,\tau)$ are equivalent, that is they have the same stable models.
} 

	\subsection{Inference in FP2} 

The ground skeptical and credulous 
consequences of finitary programs can be computed by using a ground ``relevant'' fragment of their ground instantiation \cite{DBLP:journals/ai/Bonatti04}. 
Similarly, we can reason over $\FP2$ programs by answering queries over finite and ground programs called support subprograms.

We start by defining a function  $\ssup(G,P)$ that,
for all call-safe a-goals $G$ and FP2 programs $P$, returns a representative set of supports for $G$ w.r.t.\ $P$.
More precisely, let $\ssup(G,P)$ be the set of all pairs $(\theta,s)$ such that $s$ is an acyclic support of $G$ in $P$ with global answer $\theta$.
\begin{proposition}
\label{th:support-computable}
  Let $P\in\FP2$ be a program with a call pattern $\tau$.  The
  restriction of $\ssup(G,P)$ to all $G$ that are call-safe
  w.r.t.\ $\tau$ is computable.
\end{proposition}
\begin{proof} (Sketch)
 By Corollary~\ref{strong-termination}, acyclic a-derivations of $G$ from $P$ 
  are finite and finitely many. 
  Then, it suffices to enumerate all acyclic a-derivations of $G$ from $P$.
\end{proof}

\begin{definition}
\label{def:sups}
  Let $P$ be a program and $Q$ be an atom.
  The \emph{support subprogram} $\sups(P,Q)$ for $Q$ w.r.t.\ $P$ is the set computed by the algorithm \textsc{supportSubProgram}\,$(P,Q)$ below.
  \begin{algorithm}
    \footnotesize{
    {\bf Algorithm} \textsc{supportSubProgram}\,$(P,Q)$
    \label{algsups}
    \begin{algorithmic}[1]
      \STATE $SSUP=\ssup(Q\varepsilon,P)$;
      \STATE $T=\{(Q\theta, G) \mid (\theta,G)\in SSUP \}$;
      \STATE $\bar{T}=\emptyset$;     
      \STATE $S=\emptyset$;      
      \WHILE {$T\neq\emptyset$} 
        \CHOOSE $(A,G)\in T$; 	                    
        \STATE $T=T\setminus \{(A,G)\}$;
        \STATE $\bar{T}=\bar{T}\cup \{(A,G)\}$;
        \STATE $S=S\cup \{A\leftarrow G\}$;
        \FORALL {$\naf B \in G$}
	  \STATE $SSUP=\ssup(B\varepsilon,P)$;
	  \STATE $T=T\cup (\{(B,G') \, \mid (\theta,G')\in SSUP\}\setminus\bar{T})$;
	\ENDFOR                
      \ENDWHILE       
      \RETURN $S$;
    \end{algorithmic}
    }
  \end{algorithm}
\end{definition}
%
For all $\FP2$ programs and suitably instantiated atoms $Q$, the support subprogram $\sups(P,Q)$
is finite, ground, and computable.

\begin{theorem}
\label{th:sups-terminates}
  Let $P\in\FP2$ be a program with a call pattern $\tau$.  For all
  atoms $Q$ such that $Q[\tau]$ is ground, the
  algorithm~\textsc{supportSubProgram}\,$(P,Q)$ terminates and returns
  a ground program.
\end{theorem}
\begin{proof}
(Sketch) It can be proved by simultaneous induction that the contents of $T$ are always ground, the atom $A$ selected at step 6 is always call-safe, and the input goal of $\ssup$ is call-safe therefore $\ssup$ is computable. The induction argument relies on the following observations: (i) the initial a-goal $Q\varepsilon$ is call-safe w.r.t.\ $\tau$ by hypothesis, (ii) by Lemma~\ref{ground-answers}, all the supports returned by $\ssup$ are ground if the input goal is call-safe, and (iii) ground atoms are vacuously call-safe. 
We are left to show that the loop at lines 5-12 terminates. 
Observe that the atoms occurring in $T$ during the computation belong to the following forest: The roots are the finitely many instances $Q\theta$ inserted at step 2; the children of each (ground) node $A$ are atoms $B$ occurring in the acyclic supports of $A$ and different from $A$ and its ancestors. By Corollary~\ref{strong-termination}, this tree is finitely branching; by Theorem~\ref{lemma:rec-pat-finite-path} all the paths are finite. Then the tree must be finite.
It follows that the algorithm cannot produce infinitely many different atoms $B$, and hence after a finite number of steps all the pairs $(B,G')$ at line 12 shall be already contained in $\bar T$ and the while statement terminates.
\end{proof}

\noindent
We are only left to show that  $\sups(P,Q)$ can be used to answer $Q$.
By using the properties of relevant subprograms \cite{DBLP:journals/ai/Bonatti04} and Theorem~\ref{th:support-vs-models}, we can prove that:

\begin{theorem}
\label{th:credskep-sups}
  Let $P\in\FP2$ be a program with a call pattern $\tau$ and let $Q$
  be an atom s.t.\ $Q[\tau]$ is ground. For all grounding
  substitutions $\theta$, $Q\theta$ is a credulous/skeptical
  consequence of $P$ iff it is a credulous/skeptical consequence of
  $\sups(P,Q)$.
\end{theorem}

It follows that call-safe queries are computable over FP2 programs.  In general, this property does not hold if call safeness does not hold, as proved by the following theorem that is based on an FP2 encoding of a Turing machine similar to those used in \cite{DBLP:journals/ai/Bonatti04}:

\begin{theorem}
  The problems of deciding whether an $\FP2$ program $P$
  credulously/skeptically entails an existentially quantified goal
  $\exists G$ are both r.e.-complete.
\end{theorem}

Moreover, the class of FP2 programs is decidable because the space of call patterns and recursion patterns for every given program $P$ is finite, and a simple generate and test algorithm can be used for FP2 membership checking. Then we get:
\begin{proposition}
 Deciding whether a program $P$ is in $\FP2$ is decidable.
\end{proposition}

  	\section{Extending FP2 with odd-cycles} 
	\label{composition}

By Lemma~\ref{lemma:rec-pat-no-oddcycle}, FP2 programs cannot be inconsistent nor express denials (that require odd-cycles). This restriction can be relaxed simply by composing FP2 programs with argument restricted programs \cite{DBLP:conf/iclp/LierlerL09}, that are currently the largest known decidable class of programs with the \emph{persistent CFSP property} and have no restriction on odd-cycles.
\begin{definition}\cite{DBLP:conf/iclp/BaseliceB08}
  A class of programs $\cal C$ has the \emph{computable finite semantics property} (CFSP
  for short) iff (i) for all $P$ in $\cal C$, $P$ has finitely many stable
  models each of which is finite, and (ii) there exists a computable
  function $f$ mapping each member of $\cal C$ onto its set of stable models.
  Moreover, the CFSP property is \emph{persistent} iff $\cal C$ is closed under language extensions (i.e., adding more constants or function symbols to the language of a program $P\in\cal C$ yields another program in $\cal C$).
\end{definition}
The CFSP property abstracts a number of program classes with function symbols: $\omega$-restricted programs~\cite{DBLP:conf/lpnmr/Syrjanen01},  
$\lambda$-restricted programs~\cite{DBLP:conf/lpnmr/GebserST07}, argument restricted programs, and more generally the semidecidable class of
finitely-ground programs~\cite{DBLP:conf/iclp/CalimeriCIL08}.
The persistent CFSP property is important because, under suitable hypotheses, programs with this property can be composed with finitary programs without affecting the decidability of inference \cite{DBLP:conf/iclp/BaseliceB08}. We need a preliminary result:

\hide{
\begin{figure}
\begin{center}
\fbox{\includegraphics[width=4in]{finite-semantics2}}
\end{center}
\end{figure}
 }

\begin{proposition}
\label{th:arg-restr-pcfsp}
  Argument restricted programs have the persistent $CFSP$.
\end{proposition}

The forms of composition studied in \cite{DBLP:conf/iclp/BaseliceB08} are the following. 
Let $\Def(P)$ denote the set of predicates \emph{defined in $P$}, that is, 
the set of all predicate symbols occurring in the head of some rule in $P$. 
Let $\Call(P)$ be the set of predicates \emph{called by $P$}, 
that is, the set of all predicate symbols occurring in the body of some rule in $P$.
Then we say that \emph{$P_1$ depends on $P_2$}, in symbols $P_1 \dep P_2$, if and only if
\[
\Def(P_1) \cap \Def(P_2)  =  \emptyset \,,\ 
\Def(P_1) \cap \Call(P_2)  =  \emptyset \,,\ 
\Call(P_1) \cap \Def(P_2) \neq  \emptyset \,.
\]
Moreover, \emph{$P_1$ and $P_2$ are independent}, in symbols $P_1 \ind P_2$, if and only if
\[
\Def(P_1) \cap \Def(P_2)  =  \emptyset \,,\	
\Def(P_1) \cap \Call(P_2)  =  \emptyset	 \,, 
\Call(P_1) \cap \Def(P_2)  =  \emptyset \,.
\]
Now the techniques of \cite{DBLP:conf/iclp/BaseliceB08}, based on the splitting theorem, can be easily adapted to prove the following result:
\begin{theorem}
For all programs $P$ and $Q$ such that $P$ is in FP2 and $Q$ has the persistent CFSP, if $P \dep Q$ or $P \ind Q$, then both credulous and skeptical consequences from $P\cup Q$ are decidable.
\end{theorem}
In particular, this result shows that it is theoretically possible to add the expressiveness of FP2 programs to argument restricted (actually, all finitely ground) programs.



	\section{Related work}
	\label{sec:related}
	
ASP programs with function symbols are able to encode infinite domains and recursive data structures,
such as lists, trees, XML/HTML documents, time. 
However, some restrictions are needed to keep inference decidable
and, to this end, ASP researchers have recently made several proposals 
\cite{DBLP:conf/lpnmr/Syrjanen01,DBLP:journals/ai/Bonatti04,DBLP:conf/lpar/SimkusE07,DBLP:conf/lpnmr/GebserST07,DBLP:journals/tplp/BaseliceBC09,DBLP:conf/iclp/CalimeriCIL08,DBLP:conf/lpnmr/CalimeriCIL09}. 
We will discuss  finitely-ground and FDNC programs, as they include all the other 
classes mentioned above.

\emph{Finitely-ground programs} are DLP programs with function symbols introduced in~\cite{DBLP:conf/iclp/CalimeriCIL08}. 
If we compare this class with $\FP2$ programs, we note that:
\begin{itemize}
\item Ground and nonground queries are computable for finitely-ground programs; call-safe 
queries are computable for $\FP2$ programs while, in general, nonground queries are r.e.-complete.
\item The answer sets of finitely-ground programs are computable because 
their semantics is finite. Infinite stable models are ruled out.
On the contrary, $\FP2$ programs may have infinite and infinitely many answer sets.
\item Finitely-ground programs are safe, while $\FP2$ programs admit unsafe rules.
\item Odd-cycles may occur in finitely-ground programs but not in $\FP2$ programs.
The latter can be extended with odd-cyclic predicates through composition with CFSP programs as shown in Section~\ref{composition}.
\item Deciding whether a program is finitely-ground is semidecidable, while the class of $\FP2$ programs 
is decidable.
\item Finitely-ground programs are disjunctive, while $\FP2$ is currently restricted to normal programs. 
\end{itemize}
Finitely-ground and $\FP2$ programs are not comparable due to the different recursion modes that they admit
and that make finitely-ground programs suitable for a bottom-up evaluation and $\FP2$ programs suitable 
for a top-down evaluation.  
Figures~\ref{fig:list-prg}, \ref{fig:sat-prg} and~\ref{fig:qbf-prg} and Example~\ref{ex:append} illustrate some programs that are $\FP2$ but not finitely-ground.
\begin{figure}
\caption{List processing}
\label{fig:list-prg}
\vspace{-0.15in}
\footnotesize{
\[
\begin{array}[t]{l}
  member(X,[X|Y]).\\
  member(X,[Y|Z]) \leftarrow member(X,Z).
\end{array}
\begin{array}{l}
reverse(L,R) \leftarrow reverse(L,[\,],R) . \\
  reverse([\,],R,R) . \\
  reverse([X|X_s],A,R) \leftarrow reverse(X_s,[X|A],R) .
\end{array}
\]
}
\end{figure}

\begin{figure}
\vspace{-0.2in}
\caption{SAT problem}
\label{fig:sat-prg}
\vspace{-0.15in}
\footnotesize{
\[
\begin{array}{ll}
  s(and(X,Y)) \leftarrow s(X), s(Y) . \qquad & s(not(X)) \leftarrow \naf s(X) . \\
  s(or(X,Y)) \leftarrow s(X) . & s(A) \leftarrow member(A,[p,q,r]), \naf ns(A) . \\
  s(or(X,Y)) \leftarrow s(Y) . & ns(A) \leftarrow member(A,[p,q,r]), \naf s(A) .
\end{array}
\]
}
\end{figure}

\begin{figure}
\caption{Satisfiability check for Quantified Boolean Formulas}
\label{fig:qbf-prg}
\footnotesize{
\[
\begin{array}{l}
  qbf(A,I) \leftarrow atomic(A), curr\_value(A/t,I) . \\
  qbf(or(F,G),I) \leftarrow qbf(F,I) . \\
  qbf(or(F,G),I) \leftarrow qbf(G,I) . \\
  qbf(and(F,G),I) \leftarrow qbf(F,I), qbf(G,I) . \\
  qbf(not(F),I) \leftarrow \naf qbf(F,I) . \\
  qbf(exists(X,F),I) \leftarrow qbf(F,[X/t | I]) . \\
  qbf(exists(X,F),I) \leftarrow qbf(F,[X/f | I]) . \\
  qbf(forall(X,F),I) \leftarrow qbf(F,[X/t | I]), qbf(F,[X/f | I]) . \\
  curr\_value(B,[B|L]) . \\
  curr\_value(X/V,[Y/W|L]) \leftarrow \naf X=Y, curr\_value(X/V,L) .
\end{array}
\]
}
\end{figure}

\hide{ 
$\FP2$ vs. Finitely Ground:\\
\begin{tabular}{l|l|l}
       & $\FP2$ & Finitely Ground\\
\hline
member & SI ($\pi$: $member \rightarrow \{2\}$) & NO (non e` safe)\\
append & SI ($\pi$: $append \rightarrow \{1\}$) & NO (non e` safe) \\
reverse & SI ($\pi$: $reverse \rightarrow \{1\}$) & NO (i fatti non sono ground) \\
SAT & SI ($\pi$: $s \rightarrow \{1\}, ns \rightarrow \{1\}, member \rightarrow \{2\}$) & NO (non e` safe) \\
QBF & SI ($\pi$: $qbf \rightarrow \{1\}, atomic \rightarrow \{1\}, curr_value \rightarrow \{2\}, = \rightarrow \{1\}$) & NO (non e` safe)
\end{tabular}
}

\emph{FDNC programs} \cite{DBLP:conf/lpar/SimkusE07}
achieve inference decidability by exploiting a tree-model property, by analogy with decidable fragments of
first-order logic such as description logics and the guarded fragment.
The tree-model property derives from syntactic restrictions on predicate arity and
on the occurrences of function symbols (modelled around the
skolemization of guarded formulae).  
FDNC programs can be applied to encode ontologies expressed in description
logics, and are suitable to model a wide class of planning problems.
Summarizing:
\begin{itemize}
\item Both FDNC and $\FP2$ programs may have infinite and infinitely many answer sets. 
\item Unlike $\FP2$ programs, the answer sets of FDNC programs can be finitely represented. 
\item Ground and nonground queries over FDNC programs are always computable; 
only call-safe queries are computable for $\FP2$ programs.
\item FDNC programs are safe, while $\FP2$ programs admit unsafe rules.
\item Odd-cycles may occur in FDNC programs but not in $\FP2$ programs. 
\item FDNC programs are disjunctive.
\end{itemize}
Therefore, FDNC and  $\FP2$ programs are incomparable. The 
programs in Figures~\ref{fig:list-prg}, \ref{fig:sat-prg} and~\ref{fig:qbf-prg} and in Example~\ref{ex:append} are examples of
$\FP2$ programs that are not FDNC.

	\section{Summary and conclusions}
	\label{sec:conclusions}

We have introduced FP2, a decidable class of well-behaved normal programs whose properties are orthogonal to those of the other decidable classes of ASP programs with function symbols.

Inference is decidable, too. We have shown a method based on a partial evaluation of the program w.r.t.\ a query $Q$ (algorithm \textsc{supportSubProgram}) that produces a ground program $\sups(P,Q)$ that can be fed to any ASP reasoner in order to answer $Q$. The query $Q$ needs not be ground: it can be call-safe, and it is not hard to see that the method can produce answer substitutions by unifying $Q$ with the stable models of $\sups(P,Q)$.

Note that currently this mixed top-down/ASP solving method is not intended to be an efficient implementation; it is only a proof method for decidability results. In future work the potential of top-down computations as an implementation technique should be evaluated and compared with the magic-set approach adopted in \cite{DBLP:conf/iclp/CalimeriCIL08}.

The norm-based definition of FP2 programs is actually a simplification of the (approximate) static analysis method for recognizing $U$-bounded finitary programs described in \cite{BoLPNMR01,DBLP:journals/ai/Bonatti04}. The binding propagation analysis of the old recognizer is more powerful, and we are planning to improve FP2 to cover more programs accepted with the old method.
Further interesting issues for future work comprise: a precise complexity analysis of FP2 membership checking and inference; support for disjunctive programs; more general forms of composition with persistently CFSP programs; integration with FDNC programs for more general support to odd-cycles.


\bibliography{baselice-bonatti}

\begin{thebibliography}{}

\bibitem[\protect\citeauthoryear{Baselice and Bonatti}{Baselice and
  Bonatti}{2008}]{DBLP:conf/iclp/BaseliceB08}
{\sc Baselice, S.} {\sc and} {\sc Bonatti, P.~A.} 2008.
\newblock Composing normal programs with function symbols.
\newblock See \citeN{DBLP:conf/iclp/2008}, 425--439.

\bibitem[\protect\citeauthoryear{Baselice, Bonatti, and Criscuolo}{Baselice
  et~al\mbox{.}}{2009}]{DBLP:journals/tplp/BaseliceBC09}
{\sc Baselice, S.}, {\sc Bonatti, P.~A.}, {\sc and} {\sc Criscuolo, G.} 2009.
\newblock On finitely recursive programs.
\newblock {\em TPLP\/}~{\em 9,\/}~2, 213--238.

\bibitem[\protect\citeauthoryear{Bonatti}{Bonatti}{2001}]{BoLPNMR01}
{\sc Bonatti, P.} 2001.
\newblock Prototypes for reasoning with infinite stable models and function
  symbols.
\newblock In {\em Logic Programming and Nonmonotonic Reasoning, 6th
  International Conference, LPNMR 2001}. LNCS, vol. 2173. Springer, 416--419.

\bibitem[\protect\citeauthoryear{Bonatti}{Bonatti}{2004}]{DBLP:journals/ai/Bon%
atti04}
{\sc Bonatti, P.~A.} 2004.
\newblock Reasoning with infinite stable models.
\newblock {\em Artif. Intell.\/}~{\em 156,\/}~1, 75--111.

\bibitem[\protect\citeauthoryear{Bonatti}{Bonatti}{2008}]{DBLP:journals/ai/Bon%
atti08}
{\sc Bonatti, P.~A.} 2008.
\newblock Erratum to: Reasoning with infinite stable models {[Artificial
  Intelligence 156 (1) (2004) 75-111]}.
\newblock {\em Artif. Intell.\/}~{\em 172,\/}~15, 1833--1835.

\bibitem[\protect\citeauthoryear{Bossi, Cocco, and Fabris}{Bossi
  et~al\mbox{.}}{1994}]{DBLP:journals/tcs/BossiCF94}
{\sc Bossi, A.}, {\sc Cocco, N.}, {\sc and} {\sc Fabris, M.} 1994.
\newblock Norms on terms and their use in proving universal termination of a
  logic program.
\newblock {\em Theor. Comput. Sci.\/}~{\em 124,\/}~2, 297--328.

\bibitem[\protect\citeauthoryear{Calimeri, Cozza, Ianni, and Leone}{Calimeri
  et~al\mbox{.}}{2008}]{DBLP:conf/iclp/CalimeriCIL08}
{\sc Calimeri, F.}, {\sc Cozza, S.}, {\sc Ianni, G.}, {\sc and} {\sc Leone, N.}
  2008.
\newblock Computable functions in {ASP}: {T}heory and implementation.
\newblock See \citeN{DBLP:conf/iclp/2008}, 407--424.

\bibitem[\protect\citeauthoryear{Calimeri, Cozza, Ianni, and Leone}{Calimeri
  et~al\mbox{.}}{2009}]{DBLP:conf/lpnmr/CalimeriCIL09}
{\sc Calimeri, F.}, {\sc Cozza, S.}, {\sc Ianni, G.}, {\sc and} {\sc Leone, N.}
  2009.
\newblock Magic sets for the bottom-up evaluation of finitely recursive
  programs.
\newblock In {\em LPNMR}, {E.~Erdem}, {F.~Lin}, {and} {T.~Schaub}, Eds. Lecture
  Notes in Computer Science, vol. 5753. Springer, 71--86.

\bibitem[\protect\citeauthoryear{de~la Banda and Pontelli}{de~la Banda and
  Pontelli}{2008}]{DBLP:conf/iclp/2008}
{\sc de~la Banda, M.~G.} {\sc and} {\sc Pontelli, E.}, Eds. 2008.
\newblock {\em Logic Programming, 24th International Conference, ICLP 2008,
  Udine, Italy, December 9-13 2008, Proceedings}. Lecture Notes in Computer
  Science, vol. 5366. Springer.

\bibitem[\protect\citeauthoryear{Eiter, Leone, Mateis, Pfeifer, and
  Scarcello}{Eiter et~al\mbox{.}}{1997}]{DLV}
{\sc Eiter, T.}, {\sc Leone, N.}, {\sc Mateis, C.}, {\sc Pfeifer, G.}, {\sc
  and} {\sc Scarcello, F.} 1997.
\newblock A deductive system for non-monotonic reasoning.
\newblock In {\em Logic Programming and Nonmonotonic Reasoning, 4th
  International Conference, LPNMR'97, Proceedings}. LNCS, vol. 1265. Springer,
  364--375.

\bibitem[\protect\citeauthoryear{Gebser, Schaub, and Thiele}{Gebser
  et~al\mbox{.}}{2007}]{DBLP:conf/lpnmr/GebserST07}
{\sc Gebser, M.}, {\sc Schaub, T.}, {\sc and} {\sc Thiele, S.} 2007.
\newblock Gringo : A new grounder for answer set programming.
\newblock In {\em LPNMR}, {C.~Baral}, {G.~Brewka}, {and} {J.~S. Schlipf}, Eds.
  Lecture Notes in Computer Science, vol. 4483. Springer, 266--271.

\bibitem[\protect\citeauthoryear{Gelfond and Lifschitz}{Gelfond and
  Lifschitz}{1991}]{gelfond91classical}
{\sc Gelfond, M.} {\sc and} {\sc Lifschitz, V.} 1991.
\newblock Classical negation in logic programs and disjunctive databases.
\newblock {\em New Generation Computing\/}~{\em 9,\/}~3-4, 365--386.

\bibitem[\protect\citeauthoryear{Genaim, Codish, Gallagher, and Lagoon}{Genaim
  et~al\mbox{.}}{2002}]{GCGL02}
{\sc Genaim, S.}, {\sc Codish, M.}, {\sc Gallagher, J.}, {\sc and} {\sc Lagoon,
  V.} 2002.
\newblock Combining norms to prove termination.
\newblock In {\em Verification, Model Checking, and Abstract Interpretation,
  Third International Workshop, VMCAI 2002}. LNCS, vol. 2294. Springer,
  126--138.

\bibitem[\protect\citeauthoryear{Lierler and Lifschitz}{Lierler and
  Lifschitz}{2009}]{DBLP:conf/iclp/LierlerL09}
{\sc Lierler, Y.} {\sc and} {\sc Lifschitz, V.} 2009.
\newblock One more decidable class of finitely ground programs.
\newblock In {\em ICLP}, {P.~M. Hill} {and} {D.~S. Warren}, Eds. Lecture Notes
  in Computer Science, vol. 5649. Springer, 489--493.

\bibitem[\protect\citeauthoryear{Lloyd}{Lloyd}{1984}]{DBLP:books/sp/Lloyd84}
{\sc Lloyd, J.~W.} 1984.
\newblock {\em Foundations of Logic Programming, 1st Edition}.
\newblock Springer.

\bibitem[\protect\citeauthoryear{Marek and Remmel}{Marek and
  Remmel}{2008}]{DBLP:conf/iclp/MarekR08}
{\sc Marek, V.~W.} {\sc and} {\sc Remmel, J.~B.} 2008.
\newblock On the continuity of {Gelfond-Lifschitz} operator and other
  applications of proof-theory in {ASP}.
\newblock See \citeN{DBLP:conf/iclp/2008}, 223--237.

\bibitem[\protect\citeauthoryear{Niemel{\"a} and Simons}{Niemel{\"a} and
  Simons}{1997}]{smodels}
{\sc Niemel{\"a}, I.} {\sc and} {\sc Simons, P.} 1997.
\newblock Smodels -- an implementation of the stable model and well-founded
  semantics for normal {LP}.
\newblock In {\em Logic Programming and Nonmonotonic Reasoning, 4th
  International Conference, LPNMR'97, Proceedings}. LNCS, vol. 1265. Springer,
  421--430.

\bibitem[\protect\citeauthoryear{Simkus and Eiter}{Simkus and
  Eiter}{2007}]{DBLP:conf/lpar/SimkusE07}
{\sc Simkus, M.} {\sc and} {\sc Eiter, T.} 2007.
\newblock {FDNC}: Decidable non-monotonic disjunctive logic programs with
  function symbols.
\newblock In {\em 14th Int. Conf. on Logic for Programming, Artificial
  Intelligence, and Reasoning, LPAR 2007}. Lecture Notes in Computer Science,
  vol. 4790. Springer, 514--530.

\bibitem[\protect\citeauthoryear{Syrj{\"a}nen}{Syrj{\"a}nen}{2001}]{DBLP:conf/%
lpnmr/Syrjanen01}
{\sc Syrj{\"a}nen, T.} 2001.
\newblock Omega-restricted logic programs.
\newblock In {\em LPNMR}, {T.~Eiter}, {W.~Faber}, {and} {M.~Truszczynski}, Eds.
  Lecture Notes in Computer Science, vol. 2173. Springer, 267--279.

\end{thebibliography}
\bibliographystyle{acmtrans}

\end{document}